\documentclass{article}

\usepackage{microtype}
\usepackage{graphicx}
\usepackage{subcaption}
\usepackage{booktabs} 

\usepackage{hyperref}

\usepackage[accepted]{icml2026}

\usepackage{amsmath}
\usepackage{amssymb}
\usepackage{mathtools}
\usepackage{amsthm}
\usepackage{enumerate}

\DeclareMathOperator*{\argmin}{arg\,min}

\usepackage[capitalize,noabbrev]{cleveref}

\newcommand{\x}{\mathbf{x}}
\newcommand{\z}{\mathbf{z}}

\newcommand{\dd}{\mathbf{d}}

\newcommand{\vv}{\mathbf{v}}

\newcommand{\y}{\mathbf{y}}
\newcommand{\e}{\mathbf{e}}
\newcommand{\bb}{\mathbf{b}}

\newcommand{\X}{\mathcal{X}}

\newcommand{\Y}{\mathcal{Y}}

\newcommand{\oo}{\mathcal{O}}

\newcommand{\pp}{\mathcal{P}}

\newcommand{\C}{\mathcal{C}}

\newcommand{\su}{\operatorname{supp}}
\newcommand{\dia}{\operatorname{diag}}

\theoremstyle{plain}
\newtheorem{theorem}{Theorem}[section]

\newtheorem{lemma}{Lemma}

\newtheorem{definition}{Definition}

\theoremstyle{remark}

\usepackage[textsize=tiny]{todonotes}

\icmltitlerunning{Lower Bounds for Bilevel Optimization with First-Order Oracles}

\begin{document}

\twocolumn[
  \icmltitle{Lower Complexity Bounds for Nonconvex-Strongly-Convex Bilevel Optimization with First-Order Oracles}



  \icmlsetsymbol{equal}{*}

  \begin{icmlauthorlist}
    \icmlauthor{Kaiyi Ji}{yyy}
  \end{icmlauthorlist}

  \icmlaffiliation{yyy}{Department of Computer Science and Engineering, University at Buffalo, New York, United States}

  \icmlcorrespondingauthor{Kaiyi Ji}{kaiyiji@buffalo.edu}

  \icmlkeywords{Machine Learning, ICML}

  \vskip 0.3in
]



\printAffiliationsAndNotice{}  

\begin{abstract}
Although upper bound guarantees for bilevel optimization have been widely studied, progress on lower bounds has been limited due to the complexity of the bilevel structure. In this work, we focus on the smooth nonconvex-strongly-convex setting and develop new hard instances that yield nontrivial lower bounds under deterministic and stochastic first-order oracle models. In the deterministic case, we prove that any first-order zero-respecting algorithm requires at least $\Omega(\kappa^{3/2}\epsilon^{-2})$ oracle calls to find an $\epsilon$-accurate stationary point, improving the optimal lower bounds known for single-level nonconvex optimization and for nonconvex-strongly-convex min-max problems. In the stochastic case, we show that at least $\Omega(\kappa^{5/2}\epsilon^{-4})$ stochastic oracle calls are necessary, again strengthening the best known bounds in related settings. Our results expose substantial gaps between current upper and lower bounds for bilevel optimization and suggest that even simplified regimes, such as those with quadratic lower-level objectives, warrant further investigation toward understanding the optimal complexity of bilevel optimization under standard first-order oracles.

\end{abstract}
\section{Introduction}
In this paper, we are interested in solving the following bilevel optimization problem:
\begin{align}
    &\min_{x\in\mathcal{X}} H(\x):=f(\x;\y^*(\x))\nonumber
    \\ &\text{s.t.}~\y^*(\x)=\argmin_{\y\in\mathcal{Y}} g(\x;\y),
\end{align}
where $\mathcal{X}\subset \mathbb{R}^m$ and $\mathcal{Y}\subset \mathbb{R}^n$ are nonempty closed convex sets. In this paper, we study the smooth nonconvex-strongly-convex bilevel optimization setting, where the lower-level function $g$ is smooth and strongly convex in $\y$, while the upper-level function $f$ is smooth and potentially nonconvex. This formulation captures a variety of modern applications, including meta-learning~\citep{rajeswaran2019meta}, reinforcement learning~\citep{konda2000actor,hong2020two}, robotics~\citep{wang2024imperative}, as well as communication networks and federated learning~\citep{ji2023network,tarzanagh2022fednest,huang2023achieving}.

Recent years have witnessed substantial progress in understanding the convergence and complexity of bilevel optimization. A broad class of works~\citep{ji2021bilevel,hong2020two,chen2021single,dagreou2022framework} analyzes nonconvex-strongly-convex bilevel problems under access to second-order information such as Hessian- and Jacobian-vector products. More recently, there has been growing interest in developing and analyzing \emph{fully first-order} bilevel algorithms that avoid any second-order computations~\citep{shen2023penalty,chen2025near,lu2024first,kwon2023fully,liu2020generic}.

Although upper-bound complexity analyses for bilevel optimization have been extensively studied, progress on establishing tight \emph{lower bounds} has been much slower. This difficulty largely stems from the intrinsic complexity of the general bilevel formulation. In particular, deriving meaningful lower bounds that capture the dependence on condition numbers and the target accuracy~$\epsilon$ requires carefully constructed hard instances; otherwise, one risks obtaining vacuous bounds that are no stronger than classical single-level lower bounds. 
\citet{ji2023lower} establish lower bounds for strongly-convex--strongly-convex and convex--strongly-convex bilevel problems under second-order oracle access, assuming that the hyper-objective $H(\x)$ is convex or strongly convex. Their results reveal a gap of a factor $\sqrt{\kappa}$ compared to corresponding lower bounds for min--max optimization under analogous assumptions, where $\kappa$ denotes the condition number of the lower-level function. However, their analysis is limited to the deterministic setting, and the convexity assumptions on the hyper-objective may be restrictive for general bilevel problems. 
\citealt{dagreou2024lower} derive a lower bound of $\Omega(n + \sqrt{n}\,\epsilon^{-2})$ for finite-sum nonconvex--strongly-convex bilevel problems. This bound, however, does not reflect the dependence on condition numbers and is weaker than the known lower bounds for min-max problems of the same type~\citep{zhang2021complexity}. 
More recently, \citet{kwon2024complexity} establish lower bounds for nonconvex--strongly-convex bilevel optimization under a so-called $\y^\ast$-aware stochastic first-order oracle, which returns an estimate $\hat{\y}$ that is $\epsilon$-close to the exact lower-level solution $\y^\ast$. This oracle effectively reduces the problem to one resembling single-level optimization. Nevertheless, lower bounds for standard (stochastic) first-order oracles that directly access the upper- and lower-level functions $f$ and $g$ remain an open problem.

In this paper, we take a further step toward reducing this gap by developing nontrivial lower bounds for smooth nonconvex--strongly-convex bilevel optimization under standard first-order oracle models. Our main contributions are summarized below.

\vspace{-0.2cm}

\begin{itemize}

\item \textbf{Deterministic setting.}
We construct a hard instance on which no first-order zero-respecting algorithm can find an $\epsilon$-stationary solution using fewer than 
$\Omega(\kappa^{3/2}\epsilon^{-2})$ first-order oracle calls for smooth nonconvex-strongly-convex bilevel problems. 
In comparison, the optimal lower bounds for related settings are 
$\Omega(\epsilon^{-2})$ for general smooth nonconvex single-level optimization~\citep{carmon2020lower} 
and $\Omega(\sqrt{\kappa} \epsilon^{-2})$ for smooth nonconvex-strongly-convex min--max optimization~\citep{li2021complexity}. 
Our result improves these bounds by factors of $\kappa^{3/2}$ and $\kappa$, respectively.

On the upper-bound side, \citet{chen2025near} propose a first-order penalty method achieving a convergence rate of order $\kappa^{4}\epsilon^{-2}$, which can be reduced to $\kappa^{3.5}\epsilon^{-2}$ through a naive application of Nesterov acceleration. 
However, even when compared with our lower bound, there remains a gap of order $\kappa^{2}$, indicating substantial room for future improvements.

\vspace{-0.1cm}

\item \textbf{Stochastic setting.}
We further construct an instance showing that no first-order zero-respecting algorithm can achieve an $\epsilon$-stationary solution with fewer than 
$\Omega(\kappa^{5/2}\epsilon^{-4})$ stochastic oracle calls under bounded variance assumptions.
For comparison, the lower bound for standard smooth nonconvex single-level stochastic optimization is $\Omega(\epsilon^{-4})$~\citep{arjevani2023lower}, 
and for smooth nonconvex--strongly-convex min--max optimization it is $\Omega(\kappa^{1/3}\epsilon^{-2})$~\citep{li2021complexity}. 
Our result improves upon these by factors of $\kappa^{5/2}$ and $\kappa^{13/6}$, respectively.
Compared with the $\Omega(\epsilon^{-6})$ upper bound established by \citet{kwon2024complexity}, a gap still remains.

\vspace{-0.1cm}

\item \textbf{Implications.}
Our constructions demonstrate that nontrivial lower bounds for nonconvex-strongly-convex bilevel optimization are indeed possible and are significantly stronger than the known results for single-level and min--max problems. 
Nevertheless, substantial gaps persist between current upper and lower bounds, even in this restricted setting. 
Motivated by our findings, we suggest that closing these gaps may require first studying the simpler yet meaningful case in which the lower-level function is {\bf quadratic}. 
Our lower bounds continue to apply in that regime, but obtaining tighter upper bounds in this setting remains largely unexplored and not yet well understood.
 We hope that the results presented in this paper offer valuable insights for future progress in this direction.


\vspace{-0.2cm}

\end{itemize}

\section{Related Works}

{\bf Bilevel optimization algorithms.} 
Bilevel optimization has a long history dating back to the seminal work of \citet{bracken1973mathematical}. Early studies 
\citep{hansen1992new, shi2005extended} approached bilevel programs from a constrained optimization perspective, motivating the development of KKT-based reformulations and related techniques.  
More recently, gradient-based bilevel optimization has attracted significant attention due to its efficiency and scalability in modern machine learning applications. A major class of gradient-based approaches is the family of Approximate Implicit Differentiation (AID) methods 
\citep{domke2012generic, liao2018reviving, ji2021bilevel, dagreou2022framework, yang2024tuning}, which compute the hypergradient via implicit differentiation and approximate the resulting linear system using iterative solvers.  
In contrast, Iterative Differentiation (ITD) methods 
\citep{maclaurin2015gradient, franceschi2017forward} estimate hypergradients by unrolling the lower-level optimization and applying automatic differentiation in either forward or reverse mode.  
Building upon these ideas, a number of stochastic bilevel algorithms have been developed using Neumann-series approximation 
\citep{chen2021single, ji2021bilevel}, recursive momentum techniques 
\citep{yang2021provably, guo2021randomized}, and variance-reduction mechanisms \citep{yang2021provably}.  
All such methods rely on second-order information, commonly in the form of Hessian–vector or Jacobian–vector products.  
A comprehensive overview is provided in the survey~\citep{liu2021investigating}.

Recently, growing interest has shifted slightly toward designing \emph{first-order} bilevel optimization methods that use only (stochastic) first-order oracles, thereby avoiding explicit second-order computations.  
Representative examples include penalty-based methods 
\citep{shen2023penalty, lu2024first, kwon2023fully,jiang2025beyond,chen2025near}, 
primal–dual frameworks \citep{sow2022constrained}, 
finite-difference Hessian–vector approximation techniques \citep{yang2023achieving}, 
value-function-based approaches \citep{liu2020generic, liu2021towards, liu2021value},  
barrier-based formulations \citep{liu2022bome},  
and min–max optimization based methods \citep{lu2026solving, wang2023effective}.  
These works collectively highlight the potential of first-order bilevel algorithms to achieve competitive performance while significantly reducing computational overhead.

\vspace{0.2cm}
\noindent{\bf Upper bound analysis.} 
A large body of work, including \citealt{ji2021bilevel,hong2020two,chen2021single}, studies AID- and ITD-type algorithms for nonconvex--strongly-convex bilevel optimization. Another line of research considers cases where the lower-level objective is not strongly convex; for example, \citealt{arbel2022non,liu2021towards} analyze settings in which the lower-level solution is characterized through a selection map (e.g., the output of a particular algorithm). For bilevel algorithms that rely solely on (stochastic) first-order oracles, \citealt{kwon2023fully,chen2025near} establish convergence guarantees for nonconvex--strongly-convex formulations. \citealt{shen2023penalty,chen2024finding} study algorithms under weaker structural assumptions on the lower-level problem, extending beyond strong convexity.


\vspace{0.2cm}
\noindent{\bf Lower bound analysis.} 
Foundational lower bounds for first-order optimization were established by Nemirovski and Nesterov and are presented in their textbooks~\citep{nemirovsky1992information,nesterov2018lectures}. A central concept in this theory is the notion of \emph{zero-chains}, which ensure that any zero-respecting first-order method can activate coordinates only sequentially. Recent works have significantly advanced these constructions in the context of smooth nonconvex optimization~\citep{fang2018spider,carmon2020lower,carmon2021lower,arjevani2023lower}. Building upon these developments, \citet{li2021complexity} establish lower bounds for nonconvex--strongly-convex min--max optimization. Our work is highly inspired by these results.


Lower bounds for bilevel optimization are relatively underexplored. \citet{ji2023lower} derive bounds for convex and strongly-convex bilevel problems using second-order oracles. More recently, \citet{kwon2024complexity} establish lower bounds for nonconvex-strongly-convex bilevel problems under a $\y^*$-aware stochastic oracle. \citealt{dagreou2024lower} derive a lower bound for finite-sum nonconvex-strongly-convex bilevel problems. In contrast, we provide lower bounds for nonconvex-strongly-convex bilevel optimization using standard (stochastic) first-order oracles.

{\bf A concurrent work.} As we were preparing the final draft of this paper, we became aware of a concurrent nice work by \citet{chen2025condition}, which was posted on arXiv. This work also establishes lower bounds for nonconvex-strongly-convex bilevel optimization under (stochastic) first-order oracle access, showing a larger dependence on the condition number than those of min-max and single-level minimization problems of the same type. 
Despite addressing a similar question, the constructions in the two works differ substantially. For example, the construction introduces an additional auxiliary variable $z$, whereas our construction is simpler without such variable. In the stochastic setting, \citet{chen2025condition} eliminate the coupling variable $\y$ to reduce the dimensionality, while we instead use two bounded hypercubes to control the noise variances. 


\section{Preliminaries}

\noindent {\bf Notations.} We use bold lower-case letters to denote vectors and regular lower-case letters to denote scalars. 
For a vector $\x \in \mathbb{R}^d$, we use $\x^{t}$ to denote its value at the $t^{th}$ iteration, and $x_i$ to denote its $i$th coordinate and define its support as 
$\su(\x) := \{\, i \mid x_i \neq 0 \,\}.$
We use $\|\x\|_2 = \sqrt{\sum_{i=1}^d x_i^2}$ and 
$\|\x\|_\infty = \max_{1 \le i \le d} |x_i|$ to denote the $\ell_2$ and $\ell_\infty$ norms, respectively. 
For a matrix $M \in \mathbb{R}^{m \times n}$, we use $M_{i,j}$ to denote its $(i,j)$th entry. 
We use $\|M\|_{\infty} = \max_{1 \le i \le m}\sum_{j=1}^n |M_{i,j}|$ for the matrix infinity norm and 
$\|M\|_2$ for its spectral norm. 
For a square matrix $M$, we let $\dia_m(M)$ denote the block diagonal matrix with $m$ identical copies of $M$ on the diagonal. 
We use standard asymptotic notation $\mathcal{O}(\cdot)$, $\Omega(\cdot)$, and $\Theta(\cdot)$.

\subsection{Function Class}
In this paper, we focus on the class of smooth nonconvex-strongly-convex bilevel problems that satisfy the standard assumptions used in first order bilevel optimization.

\begin{definition}\label{def:fc}
Given $L_f,L_g\geq\mu>0$, $C\geq 0$ and $\Delta>0$, define $\mathcal{F}(L_f,L_g,\mu,\Delta)$ to be the set of function pairs $\{f,g\}$ such that  $f:\mathcal{X}\times \mathcal{Y} \rightarrow \mathbb{R}$ and $g:\mathcal{X}\times \mathcal{Y} \rightarrow \mathbb{R}$ for some nonempty closed convex sets $\mathcal{X}\subset \mathbb{R}^m$ and $\mathcal{Y}\subset \mathbb{R}^n$ for all $m,n\in\mathbb{N}$, which satisfy the following assumptions:
\begin{enumerate}
    \item Functions $f, g$ are continuously differentiable,  $L_f$ and $L_g$-smooth respectively, jointly in $(\x,\y)$ over $\X\times\Y$.
    \item For every $(\x,\y) \in \X\times\Y$, there exists a numerical constant $C\geq 0$ such that 
        $\|\nabla_y f(\x,\y)\|_2 \le C$.

    \item For every $\x \in \X$, $g(\x,\cdot)$ is $\mu$-strongly-convex in $\y$, that is, for any $\y_1,\y_2\in\mathcal{Y}$,
    \begin{align*}
        g(\x;\y_1) \geq& g(\x;\y_2) + \langle \nabla_\y g(\x;\y_2), \y_1-\y_2 \rangle \\+& \frac{\mu}{2} \|\y_1-\y_2\|^2_2.
    \end{align*}
    
        \vspace{-0.5cm}

    \item There exists a numerical constant $\rho\geq 0$ such that the second-order derivatives $\nabla^2_{\x,\y} g$ and $\nabla^2_{\y,\y} g$ are well-defined and $\rho$-Lipschitz jointly in $(\x,\y)$ for all $(\x,\y) \in \X\times\Y$.
    \item $H(\mathbf{0})-\min_{\x\in\X}H(\x)\leq \Delta$, where $H(\x) := f(\x;\y^*(x))$ is the hyper-objective function. 
\end{enumerate}
Note that for items~2 and~4, we only require the existence of numerical constants $C,\rho=\oo(1)$.
\end{definition} 
Although we do not explicitly specify $L_g$ and $L_f$ in this paper, we are primarily interested in the regime where these constants are independent of the strong convexity parameter $\mu$ and the target accuracy $\epsilon$.


\subsection{Algorithm Class}
We focus on algorithms that solve bilevel optimization problems using (stochastic) first order oracles. 
For clarity of presentation, we first define the (stochastic) first-order oracles considered in this work.
\begin{definition}[Deterministic first-order oracle]
    The deterministic first-order oracle of a differentiable function $f:\X\rightarrow \mathbb{R}$ is a mapping $O: \x \mapsto (f(\x),\nabla f(\x))$ for $\x\in\X$.
\end{definition}

\begin{definition}[Stochastic first-order oracle]
     The stochastic first-order oracle of a differentiable function $f:\X\rightarrow \mathbb{R}$ is a mapping $O: \x \mapsto (f(\x),G_f(\x;\xi))$ for $\x\in\X$, where $\xi$ is a random variable satisfying $\mathbb{E}_\xi \left[G_f(\x;\xi)\right]=\nabla f(\x)$ and $\mathbb{E}_\xi\|G_f(\x;\xi)-\nabla f(\x)\|^2_2\leq \sigma_f^2$.
\end{definition}
Note that the algorithms rely on first-order oracles for both the upper- and lower-level objectives $f$ and $g$. 
In the stochastic setting, we assume for simplicity that the variances of the stochastic first-order oracles are identical, i.e., $\sigma_f = \sigma_g = \sigma$. 
We further focus on first-order bilevel algorithms that satisfy the following zero-respecting property:

\begin{definition}[Algorithm class]\label{de:alc}
    For upper- and lower-level objective functions $f:\X\times\Y\rightarrow \mathbb{R}$ and $g:\X\times\Y\rightarrow \mathbb{R}$ and their  first-order oracles $O_f:(\x,\y) \mapsto (f(\x;\y),\nabla f(\x;\y))$ and $O_g: (\x,\y) \mapsto (g(\x;\y),\nabla g(\x;\y))$, the $(t+1)$-th iterate $(\x^{t+1},\y^{t+1})$ satisfies:
    \begin{align}\label{eq:span}
        \x^{t+1} \in & \bigg\{\pp_{\X}(\mathbf{u}): \su(\mathbf{u}) \subset \bigcup_{0\leq i \leq t} (\su(\x^i)\cup \nonumber
        \\&\su (\nabla_\x f(\x^i;\y^i))\cup \su (\nabla_\x g(\x^i;\y^i) )
        \bigg\}; \nonumber
        \\ 
        \y^{t+1} \in & \bigg\{\pp_{\Y}(\mathbf{v}): \su(\mathbf{v}) \subset \bigcup_{0\leq i \leq t} (\su(\y^i)\cup \nonumber
        \\&\su (\nabla_\y f(\x^i;\y^i))\cup \su (\nabla_\y g(\x^i;\y^i) )  \bigg\}.
    \end{align}
   A similar definition applies in the stochastic setting, where the gradients $\nabla f$ and $\nabla g$ are replaced by their corresponding stochastic first-order oracles.
\end{definition}
Note that the subspaces defined in \Cref{eq:span} permit both simultaneous and alternating updates of $\x$ and $\y$, thereby including single-loop and double-loop bilevel optimization algorithms. Consequently, the algorithm class introduced in \Cref{de:alc} covers all existing first-order bilevel optimization methods, including but not limited to penalty-based approaches~\citep{shen2023penalty,lu2024first}, primal–dual methods~\citep{sow2022constrained}, finite-difference Hessian–vector–approximation methods~\citep{yang2023achieving}, value-function-based approaches~\citep{liu2020generic, liu2021towards, liu2021value}, and barrier-based methods~\citep{liu2022bome}.



\vspace{-0.2cm}
\section{Lower Bounds in Deterministic Setting}
\vspace{-0.1cm}

\subsection{Useful Techniques for Lower-Bound Construction}

In this paper, we focus on the bilevel optimization setting where the lower-level function $g(\x;\y)$ is strongly convex in $\y$, while the upper-level function $f(\x;\y)$ is smooth and possibly nonconvex. For this reason, our constructions draw on key techniques and components from the worst-case instances of \citealt{nesterov2018lectures} for smooth strongly convex functions and \citealt{carmon2020lower} for smooth nonconvex functions. Their core idea is to make sure their instances satisfy the following notion of zero-chain property: 
\begin{definition}[Zero-chain]
A function $f:\X\subset\mathbb{R}^d\to\mathbb{R}$ is a (first-order) zero-chain if for every $1\le i\le d$,
\begin{align*}
   & \operatorname{supp}(\x) := \{i : x_i \neq 0\} \subset \{1,\ldots,i-1\}
\\&\qquad\Longrightarrow\;\;
\operatorname{supp}(\nabla f(\x)) \subset \{1,\ldots,i\}.
\end{align*}
\end{definition}
\vspace{-0.3cm}

Consider running a first-order algorithm on a zero-chain function, starting from the initialization 
$\x = 0$, and assume access to a deterministic
first-order oracle. By the zero-chain property, each iteration can introduce at most one new
nonzero coordinate of $\x$—that is, each iteration ``activates'' at most one additional coordinate.
Consequently, after $t$ iterations we must have 
$\operatorname{supp}(\x^t)\subset\{1,\ldots,t\}.$ Therefore, if a good solution requires that at least $T$ coordinates be discovered, then any
deterministic first-order method must take at least $T$ iterations, which yields a lower bound
of order $T$ on the algorithm's complexity.

Following this strategy, \citealt{nesterov2018lectures} and \citealt{carmon2020lower} provide the following key components for their constructions in strongly-convex and nonconvex settings: 

\vspace{-0.3cm}

\begin{itemize}
    \item {\bf Tri-diagonal matrix $A$.} Following \citealt{nesterov2018lectures,li2021complexity},  we use the following tri-diagonal 1-D discrete Laplacian matrix $A\in\mathbb{R}^{n\times n}$ to construct the strongly-convex lower-level instance:
\begin{align}\label{Amatrix}
  A:=  \begin{bmatrix}
 1& -1 &  &  & \\
 -1& 2  &-1 &  & \\
 &   \ddots & \ddots & \ddots  & \\
  && -1& 2 & -1 \\
  &  & & -1 & 1\\ 
\end{bmatrix},
\end{align}
where it is verified that $A$ is positive semidefinite and $\|A\|_2\leq 4$. Due to its tri-diagonal nature, it is easily verified that if $\su(x)\subset{1,...,i-1}$, then $Ax\subset\{1,...,i\}$. In other words, if a vector has nonzero entries only at its first $i-1$ coordinates, then multiplying it by $A$ can activate at most one additional coordinate, namely the $i$-th one.
\item {\bf $\Psi(\cdot)$ and $\Phi(\cdot)$ hardness functions.} Following the construction in \citealt{carmon2020lower}, we employ the component functions $\Psi(x):\mathbb{R} \rightarrow \mathbb{R}$ and $\Phi(x):\mathbb{R} \rightarrow \mathbb{R}$ defined below.
\begin{align}\label{eq:phipsi}
\Psi(x) := &
\begin{cases}
0, & x \le \tfrac{1}{2}, \\[6pt]
\exp\!\left( 1 - \frac{1}{(2x-1)^2} \right), & x > \tfrac{1}{2},
\end{cases} \nonumber
\\ 
\Phi(x) :=& \sqrt{e}\, \int_{-\infty}^{x} e^{-\tfrac{1}{2} t^{2}} \, dt,
\end{align}
which have the following key properties that will be used in our analysis. 
\begin{lemma}[\citealt{carmon2020lower}, Lemma 1]\label{le:phipsi}
The functions $\Phi$ and $\Psi$ satisfy
\begin{enumerate}
  \item For all $x \le \tfrac{1}{2}$ and $k \in \mathbb{N}$, we have
      $\Psi^{(k)}(x) = 0$, where $\Psi^{(k)}$ denotes the $k^{th}$-order derivative. 

  \item For all $x \ge 1,|y| < 1$, we have $\Psi(x)\,\Phi'(y) > 1$.

  \item Both $\Psi$ and $\Phi$ are infinitely differentiable. For all $k \in \mathbb{N}$, it holds that
  \begin{align*}
      \sup_x \bigl|\Psi^{(k)}(x)\bigr|
      \le \exp\!\left(\frac{5k}{2}\log(4k)\right)
      \\
      \sup_x \bigl|\Phi^{(k)}(x)\bigr|
      \le \exp\!\left(\frac{3k}{2}\log\!\frac{3k}{2}\right).
  \end{align*}

  \item The functions and derivatives $\Psi$, $\Psi'$, $\Phi$, $\Phi'$ are nonnegative and bounded, with 
    $  0 < \Psi < e,
      0 < \Psi' < \sqrt{\tfrac{54}{e}},
      0 < \Phi < \sqrt{2\pi e},
      0 < \Phi' < \sqrt{e}.$
\end{enumerate}
\end{lemma}
\vspace{-0.3cm}

\citealt{carmon2020lower} use a construction of $f(\x)=\sum_i \big[\Psi(-x_{i-1})\Phi(-x_{i})-\Psi(x_{i-1})\Phi(x_{i})\big]$, which together with $\Psi^\prime(0)=\Psi(0)=0$,  ensures the zero-chain property that if $\x\subset \{1,...,i-1\}$, then $\nabla f(\x)\subset \{1,...,i\}$. Furthermore, as we will show later, the boundedness of $\Psi$, $\Psi'$, $\Phi$, and $\Phi'$ is crucial for constructing a valid worst-case instance within the bilevel class $\mathcal{F}(L_f, L_g, \mu, C, \Delta)$.
\end{itemize}

\vspace{-0.2cm}
\subsection{Main Result: A Lower Bound on First-Order Oracle Complexity}
The following theorem establishes a complexity lower bound for deterministic first-order bilevel algorithms.
\begin{theorem}\label{th:deter}
For any $L_f, L_g, \mu, \Delta, \epsilon > 0$ satisfying $\kappa = L_g / \mu \ge 1$ and $\frac{\Delta}{L_f}=\oo(1)$, 
there exist functions 
$f : \mathbb{R}^m \times \mathbb{R}^n \to \mathbb{R}$ and 
$g : \mathbb{R}^m \times \mathbb{R}^n \to \mathbb{R}$ 
such that $\{f, g\} \in \mathcal{F}(L_f, L_g, \mu, \Delta)$ for some $m, n \in \mathbb{N}$ with their deterministic first-order oracles. 
For any first-order bilevel algorithm of the form in \Cref{de:alc}, in order to find an $\epsilon$-accurate stationary point $\x$ such that 
$\|\nabla H(\x)\|_2 < \epsilon$, the algorithm must use at least
   $ \frac{C_0 \Delta L_f \kappa^{3/2}}{\epsilon^2}$
oracle calls, 
where $H(\x) = f(\x; \y^*(\x))$ with $y^*(\x)=\argmin_\y g(\x;\y)$ is the hyper-objective, and $C_0$ is a numerical constant. 
\end{theorem}
\citealt{carmon2020lower} establish a lower bound of $\Omega(1/\epsilon^{2})$ for smooth nonconvex optimization, and \cite{li2021complexity} proves a lower bound of $\Omega(\sqrt{\kappa}/\epsilon^{2})$ for smooth nonconvex-strongly-concave min-max optimization. 
Both results can be viewed as special cases of smooth nonconvex-strongly-convex bilevel optimization, for which we obtain in \Cref{th:deter} a much larger lower bound of $\Omega(\kappa^{3/2}/\epsilon^{2})$. 
This demonstrates that bilevel optimization is provably more challenging than min-max optimization. 
This observation is consistent with the fundamental hardness comparison for smooth strongly-convex--strongly-convex bilevel problems established in \citealt{ji2023lower}.

\vspace{-0.2cm}

\begin{figure*}[t]
  \centering
   \includegraphics[width=0.9\linewidth]{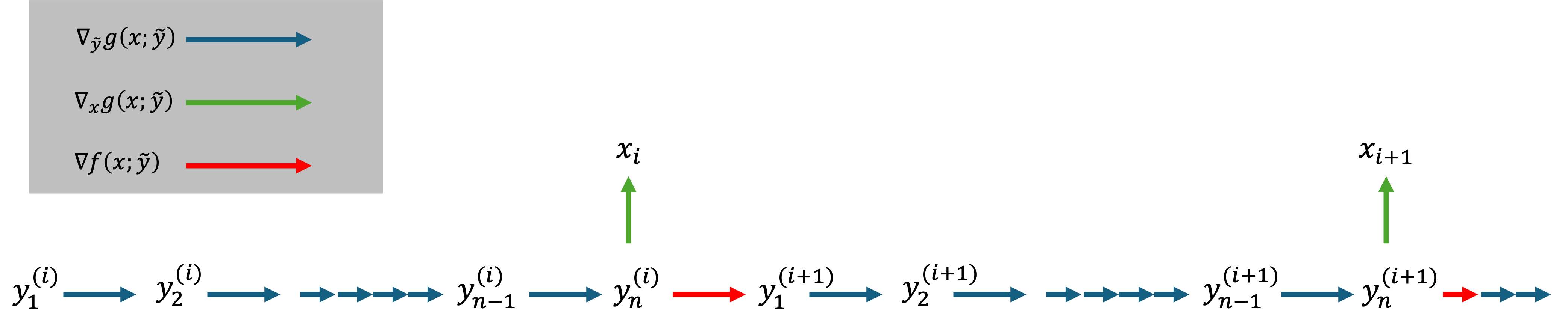}\


   \caption{An illustration of the zero-chain for our constructed instance in \cref{obj:construction} for nonconvex-strongly-convex bilevel optimization. } 
   \label{fig:zero_chain}
\end{figure*}


\subsection{Proof Outline for Deterministic Lower Bound}

We consider the following worst-case instance. For notational simplicity, define $x_0\equiv\frac{\lambda}{C_lM_{n,n}}$.
\begin{align}\label{obj:construction}
f(\x;\mathbf{\widetilde y}) =& \sum_{i=1}^T \frac{\lambda^2L_f}{L}\bigg[
\Psi\Big(-\frac{C_l}{\lambda}y^{(i-1)}_n\Big)\Phi\Big(-\frac{C_r}{\lambda}y^{(i)}_1\Big) \nonumber
\\-&\Psi\Big(\frac{C_l}{\lambda}y^{(i-1)}_n\Big)\Phi\Big(\frac{C_r}{\lambda}y_1^{(i)}\Big) 
\bigg]  
\nonumber  \\
g(\x;\mathbf{\widetilde y}) = & \sum_{i=0}^T \Big[ \frac{L_gn^2}{2(4n^2+1)} 
(\y^{(i)})^\top \Big(\frac{1}{n^2}I_n + A\Big)\y^{(i)} 
\nonumber\\&- L_g(\bb_x^{(i)})^\top\y^{(i)}
\Big],
\end{align}
where $\x=[x_1,...,x_{T}]\in\mathbb{R}^{T}$ is the upper-level variable, $\mathbf{\widetilde y} = [\y^{(0)},\y^{(1)},....,\y^{(T)}]$ with each $\y^{(i)}\in\mathbb{R}^n$ is the lower-level variable, $y^{(i)}_j$ returns the $j^{th}$ coordinate of $\y^{(i)}$, and the dimension $n=\big\lfloor \sqrt{\frac{L_g-\mu}{4\mu}} \big\rfloor$, and the design of $\bb_x^{(i)}$ is most critical, which is given by 
\begin{align*}
    \bb_x^{(i)} &= [0,0,....,x_i ] = x_i  \e_n,
\end{align*}
where $\e_i$ denotes the $i^{\text{th}}$ standard basis vector, whose sole nonzero entry equals $1$. For simple presentation, the numerical constants $C_l, C_r, L$ and the parameter $\lambda$ will be specified at a later stage.

\vspace{0.2cm}
\noindent{\bf Validation of our constructed instance.} We first verify that our constructed instance belongs to the function class $\mathcal{F}(L_f, L_g, \mu, \Delta)$.

\begin{enumerate}
    \item 
First, we need to verify $g(\x;\cdot)$ is $\mu$-strongly convex. Since the matrix $A$ is positive semidefinite, it can be verified that $\nabla^2 g(\x;\cdot)=\frac{L_gn^2}{4n^2+1}\dia_{T+1}(A+\frac{1}{n^2})$. Let $M:=\dia_{T+1}\{A\}$.  For any vector $z \in \mathbb{R}^{n(T+1)}$, write it as a block vector $\z=[\z_1,\z_2,...,\z_m], \z_i \in \mathbb{R}^n$, we have
$\z^\top M \z = \sum_{i=1}^m \z_i^\top A \z_i.$ 
Since $A$ is positive semidefinite, each term $\z_i^\top A \z_i \ge 0$, so the sum is nonnegative. Hence $\z^\top M \z \ge 0$ for all $\z$, and therefore $M$ is positive semidefinite. This further implies that $\|\nabla^2 g(\x;\cdot)\|_2\geq \frac{L_g}{4n^2+1}.$ Given that $n=\big\lfloor \sqrt{\frac{L_g-\mu}{4\mu}} \big\rfloor\leq \sqrt{\frac{L_g-\mu}{4\mu}}$, we have $\frac{L_g}{4n^2+1}\geq \mu$. This validates that $g(\x;\widetilde \y)$ is $\mu$-strongly convex in $\widetilde \y$.

\item Next, we validate the smoothness of $f$ and $g$ functions:
\begin{itemize}
    \item For the lower-level function $g(\x;\widetilde\y)$, it follows from \cref{obj:construction} that $\|\nabla^2_{\widetilde\y} g(\x;\widetilde\y)\|_2\leq \frac{L_gn^2}{4n^2+1}(\frac{1}{n^2}+4)=L_g$, $\|\nabla^2_{\x,\widetilde \y}g(\x;\widetilde\y)\|_2=L_g, \nabla^2_\x g(\x;\widetilde\y) = 0$ for any $\x,\widetilde \y$, and hence $g(\x;\widetilde\y)$ is $L_g$-smooth.  
    \item For the upper-level function $f(\x;\widetilde\y)$, note that $\nabla^2_{\widetilde\y} f(\x;\widetilde\y)= \frac{L_f}{L} M$, where $M\in\mathbb{R}^{n(T+1)\times n(T+1)}$ is a tri-diagonal matrix, where the absolute value of 
    each nonzero element is bounded by some numerical constant, due to the fact that $C_r$ and $C_l$ are numerical constants, and that 
the functions $\Phi$ and $\Psi$, together with their derivatives, are bounded by numerical constants, as shown in item~3 of \Cref{le:phipsi}. Then, we have  
    $\|M\|_2\leq C_M$ for some numerical constant $C_M>0$. Thus, choosing $L=C_M$ yields  $\|\nabla^2_{\widetilde\y} f(\x;\widetilde\y)\|_2\leq L_f$. Since $f(\x;\widetilde\y)$ depends only on $\widetilde \y$, it is thus $L_f$-smooth.
\end{itemize}

\item Next, we need to show that the gradient norm $\|\nabla_{\widetilde\y} f(\x;\widetilde\y)\|_2$ is bounded by a numerical constant that is independent of both $T$ and $n$. 
This step is particularly challenging. 
For example, the previous lower bound in~\citealt{ji2023lower} circumvents this requirement by exploiting the strong convexity of the hyper-objective to guarantee gradient boundedness during the optimization process. 
However, that strategy applies only to the strongly-convex–strongly-convex setting and may not extend well to nonconvex or stochastic regimes. 
Moreover, another lower bound in~\citealt{kwon2024complexity} sets the upper-level function as a scalar $y$, which ensures that the gradient norm remains bounded by a constant.

For our construction in \cref{obj:construction}, it can be obtained that $\|\nabla_{\widetilde\y} f(\x;\widetilde\y)\|_2 = \frac{\lambda L_f}{L} \vv$, 
where $\vv \in \mathbb{R}^{n(T+1)}$ has at most $2(T+1)$ nonzero entries at coordinates 
$kn+1$ for $k=0,\ldots,T$ and $jn$ for $j=1,\ldots,T+1$. 
Moreover, the absolute value of each nonzero entry is bounded by a positive numerical constant, owing to the fact that $C_r$ and $C_l$ are numerical constants and that $\Psi$, $\Psi'$, $\Phi$, and $\Phi'$ are all bounded (\Cref{le:phipsi}, item 3). Therefore, we have $\|v\|\leq C_0\sqrt{T}$ for some numerical constant $C_0$. Thus, we have $\|\nabla_{\widetilde\y} f(\x;\widetilde\y)\|_2\leq \frac{C_0L_f}{L} \lambda\sqrt{T}$. As will be seen later, $T$ is chosen such that $\lambda\sqrt{T}\leq \sqrt{\frac{\Delta L}{12L_f}}$, which, together with $\frac{\Delta}{L_f}=\oo(1)$, implies that $\|\nabla_{\widetilde\y} f(\x;\widetilde\y)\|_2=\oo(1)$.

\end{enumerate}


\vspace{-0.2cm}

\noindent{\bf Zero-chain properties and iterate subspaces.} We initialize $\x$ and $\widetilde\y$ to be $\mathbf{0}$. 
Then, based on the tri-diagonal structure of $A$ and the properties of $\Psi$ function in \Cref{le:phipsi} (item 1), it can be quickly verified from our construction in \cref{obj:construction} that 

\vspace{-0.3cm}

\begin{itemize}
\item At the first iteration, $y_n^{(0)}$ becomes activated, because $x_0 \neq 0$ and $\partial g /\partial y_n^{(0)}=-L_gx_0$. Thus, at the second iteration, $y_1^{(1)}$ becomes activated due to the zero-chain property of the $f(\x;\widetilde\y)$ function.

    \item Suppose the iterates have begun updating $\y^{(i)}$ but have not yet reached $y_n^{(i)}$ (i.e., $y_n^{(i)} = 0$) for some $i \ge 1$. 
This implies that $y_n^{(j)} = 0$ for all $j \ge i$. 
Then, by \Cref{le:phipsi} (item~1), it is verified that for all $j \ge i$,

\vspace{-0.2cm}

{\small
    \begin{align*}
        \frac{\partial f(\x;\widetilde\y)}{\partial y_1^{(j+1)}} &= 
-\frac{C_r\lambda L_f}{L}\Psi\Big(-\frac{C_l}{\lambda}y^{(j)}_n\Big)\Phi^\prime\Big(-\frac{C_r}{\lambda}y^{(j+1)}_1\Big) 
\\&- \frac{C_r\lambda L_f}{L}\Psi\Big(\frac{C_l}{\lambda}y^{(j)}_n\Big)\Phi^\prime\Big(\frac{C_r}{\lambda}y_1^{(j+1)}\Big) = 0,
    \end{align*} }

    \vspace{-0.1cm}

   which, together with the structure of the lower-level function and the condition $x_j = 0$ for all $j \ge i$, implies that $\y^{(j)}=\mathbf{0}$ for all $j\geq i+1$. This property is crucial because it preserves the zero-chain structure along the sequence $\{\y^{(i)}\}_{i=1}^T$ and ensures that advancing from one adjacent $\y$-iterate to the next necessarily requires at least $n$ iterations.
\item   Suppose the iterates have begun updating $\y^{(i)}$ but have not yet reached $y_n^{(i)}$ (i.e., $y_n^{(i)} = 0$) for some $i \ge 1$. Then, for all $j \ge i$, the gradient of $g(\x;\widetilde\y)$ with respect to $x_j$ is given by $-y^{(j)}_n$. 
As a consequence, the coordinate $x_j$ won't be activated until $y_n^{(j)}$ is activated.

\end{itemize}
\vspace{-0.2cm}

Based on the above analysis, it can be derived that at any iteration $Kn+k$ with $K=0,...,T-1$ and $k=1,...,n$, 
\begin{align}\label{subspaces}
    \su(\y^{(i)}) &\subseteq \{1,...,n\}, \quad i\leq K \text{ and } i\neq 0 \nonumber
    \\ \su(\y^{(K+1)}) &\subset \{1,...,k\} \nonumber
    \\ \su(\y^{(i)}) &= \emptyset, \quad i>K+1 \nonumber
    \\ \su(\x) &\subset \{0,...,K\}.
\end{align}
Accordingly, to activate all coordinates of $\x$, one must perform at least $Tn$ iterations in total.

\vspace{0.2cm}
\noindent{\bf The hyper-objective function and its key properties.} First, we can verify that the lower-level solutions are given by
\begin{align*}
    (\y^{(i)})^* = \underbrace{\frac{4n^2+1}{n^2}\left(\frac{1}{n^2}I_n + A\right)^{-1}}_{M}\bb_{x}^{(i)}.
\end{align*}
The hyper-objective $H(x):=f(x;\mathbf{\widetilde y}^*)$ is then given by 
\begin{align*}
    H(\x)=& \sum_{i=1}^T \frac{\lambda^2L_f}{L}\Big[
\Psi\big(-\frac{C_l}{\lambda}M_{n,n}x_{i-1}\big)\Phi\big(-\frac{C_r}{\lambda}M_{1,n}x_i\big) 
\\&- \Psi\big(\frac{C_l}{\lambda}M_{n,n}x_{i-1}\big)\Phi\big(\frac{C_r}{\lambda}M_{1,n}x_i\big)
\Big]. 
\end{align*}
Note that the above definition of $H(\x)$ involves the quantities $M_{n,n}$ and $M_{1,n}$, whose behaviors are characterized in the following lemma.
\begin{lemma}\label{le:M}
Let $A\in\mathbb{R}^{n\times n}$ be the tri-diagonal matrix defined by \Cref{Amatrix} and define $S := \bigl(A + \tfrac{1}{n^{2}} I_n\bigr)^{-1}$.
Then for every integer $n\ge1$,
\begin{align}\label{eq:cC}
c\,n \le S_{1,n},S_{n,n} \le C\,n,\nonumber
\\ 
c := 1-\frac{\pi^2}{12},\ 
C := 1+\frac{\pi^2}{12}.
\end{align}
\end{lemma}
\vspace{-0.3cm}

Based on \Cref{le:M} and $4\leq \frac{4n^2+1}{n^2}\leq 5$, it can be derived that $4c n \leq M_{1,n},M_{n,n}\leq 5 Cn$, where $c$ and $C$ are given by \cref{eq:cC}. Thus, choose numerical constants $C_l$ and $C_r$ such that 
\begin{align}\label{eq:tildeC}
\frac{C_l M_{n,n}}{n} \;=\; \frac{C_r M_{1,n}}{n} \;=\; \widetilde C,
\end{align}
where $\widetilde C$ is a numerical constant. Then, we use the following lemma to provide a lower bound on the gradient norm when the algorithm has not yet reached the end of the chain.

\begin{lemma}\label{le:grad}
    If $|x_i|<\frac{\lambda}{\widetilde C n}$ for some $i\leq T$. Then, we have $\|\nabla H(\x)\|_2\geq \frac{\lambda L_f \widetilde C n}{L}$.
\end{lemma}
The following lemma provides the bound on the optimality gap of the hyper-objective function $H(\x)$:
\begin{lemma}\label{le:function}
   The hyper-objective function $H(\x)$ satisfies $ H(\mathbf{0}) - \inf_{\x}H(\x) \leq \frac{12\lambda^2L_fT}{L}$. 
\end{lemma}
\vspace{-0.3cm}

Based on all the above auxiliary lemmas, we begin to prove our main theorem. 
\vspace{-0.1cm}

\begin{proof}[{\bf Proof of \Cref{th:deter}}]
    First note that if $x_T=0$, based on \Cref{le:grad}, we have that $ \|\nabla H(\x)\|_2 \geq \frac{\lambda L_f \widetilde C n}{L}.$ 
Choosing $\lambda=\frac{\epsilon L} {L_f\widetilde C n }$ guarantees $\|\nabla H(\x)\|_2 \geq \epsilon$. Then, we need to verify that $H(\mathbf{0}) - \inf_{\x}H(\x) \leq \Delta$. Based on \Cref{le:function}, we have that 
$H(\mathbf{0}) - \inf_{\x}H(\x)  \leq    \frac{12\lambda^2L_fT}{L}$,
which, by setting $T=\left\lfloor  \frac{\Delta L}{12\lambda^2L_f}  \right \rfloor$, guarantees that $H(\mathbf{0}) - \inf_{\x}H(\x) \leq \Delta$. 

Based on the subspace analysis in \cref{subspaces}, we have that $x_T=0$ if $t< Tn$, and hence $\|\nabla H(\x^t)\|_2\geq \epsilon$. Recall that $n=\big\lfloor \sqrt{\frac{L_g-\mu}{4\mu}} \big\rfloor$. Thus, to achieve an $\epsilon$-accurate stationary solution, there are at least 
\begin{align*}
    Tn = \frac{c_0\Delta n^3}{\epsilon^2} = \frac{\Delta L n }{ 12 L_f} \frac{L_f^2\widetilde C^2 n^2}{\epsilon^2 L^2} = \frac{C_0\Delta L_f \kappa^{\frac{3}{2}}}{\epsilon^2}
    \end{align*}
oracle calls, where $c_0$ is some numerical constant. Then, the proof is complete. 
\end{proof}

\vspace{-0.2cm}
\section{Lower Bounds in Stochastic Setting}
In this section, we provide a lower bound for stochastic  first-order oracles. We first introduce several important definitions and lemmas from \citealt{arjevani2023lower}, serving as the foundation for our constructions in the stochastic setting. 
\subsection{Auxiliary Definitions and Lemmas}
Following~\citealt{arjevani2023lower}, to establish a lower bound in the stochastic setting, 
we adopt the notion of a probability-$p$ zero-chain.
\begin{definition}[Probability-$p$ zero-chain]
A function $f : \mathcal{X} \to \mathbb{R}$ with a stochastic first-order oracle 
$O : \x \mapsto (f(\x), G_f(\x;\xi))$ is a probability-$p$ zero-chain if
\begin{align*}
&\operatorname{supp}(\x) \subset \{1,\ldots,i-1\} 
\\
&\Longrightarrow\quad
\begin{cases}
\mathbb{P}\!\big( \operatorname{supp}(G_f(\x;\xi)) \not\subset \{1,\ldots,i-1\} \big) \le p, \\
\mathbb{P}\!\big( \operatorname{supp}(G_f(\x;\xi)) \subset \{1,\ldots,i\} \big) = 1.
\end{cases}
\end{align*}
\end{definition}
The above definition implies that at each iteration, a new coordinate $i$ becomes activated 
(i.e., the iterate acquires a nonzero entry at coordinate $i$) with probability $p$. The following lemma (which is an adapted version from \citealt{li2021complexity}) provides a recipe for constructing a probability-$p$ zero-chain based on a given zero-chain.
\begin{lemma}[Lemma~3 in \citealt{arjevani2023lower}]\label{le:pchain}
Let $f : \mathcal{X} \to \mathbb{R}$ be a zero-chain on $\mathcal{X} \subset \mathbb{R}^T$. 
For $\x \in \mathcal{X}$, let 
$i^*(\x) := \inf\{\, i \in [T] : x_i = 0 \,\}$ 
be the next coordinate to activate.  
For $p \in (0,1]$, define the stochastic gradient estimator $G_f(\x;\xi)$ coordinate-wisely by
\[
[G_f(\x,\xi)]_i :=
\begin{cases}
\dfrac{\xi}{p}\, \nabla_i f(\x), & \text{if } i = i^*(\x),\\[6pt]
\nabla_i f(\x), & \text{otherwise},
\end{cases}
\]
where $\xi \sim \mathrm{Bernoulli}(p)$.  
Suppose there exists $G<\infty$ such that $\|\nabla f(\x)\|_\infty \le G$ for all $\x \in \mathcal{X}$.  
Then, the oracle 
$O: \x \mapsto (f(\x), G_f(\x,\xi))$
is a stochastic first-order oracle with bounded variance $\sigma^2 \le G^2 (1-p)/p$.  
Moreover, $f$ with oracle $O$ is a probability-$p$ zero-chain.
\end{lemma}
\Cref{le:pchain} allows us to build a probability-$p$ zero-chain based on the zero-chain we establish in \cref{obj:construction} and \Cref{fig:zero_chain}. However, as also noted by \citealt{li2021complexity} for min-max  problems, one main challenge lies in the unboundedness of the iterates $\x$ and $\widetilde \y$, such that the gradient norm of the lower-level function $\|\nabla g(\x;\widetilde\y)\|_\infty$ is unbounded. To this end, \citealt{li2021complexity} modify the quadratic components in their deterministic worst-case instance and introduce two bounded hypercubes as the domains for $\x,\y$:
$\C_{R_x}^m := \{\x\in \mathbb{R}^m : \|\x\|_\infty \le R_x \} $ and $
\C_{R_y}^n := \{\y\in \mathbb{R}^n : \|\y\|_\infty \le R_y \}$,
where $R_x$ and $R_y$ are chosen so that the variance of the stochastic oracle is bounded by $G$. 
Interestingly, unlike \citealt{li2021complexity}, which must revise the quadratic components in their deterministic construction,  our deterministic instance in \cref{obj:construction} can be used directly, provided that the domain radius $R_x$ and $R_y$ are properly selected, as seen in our analysis later. 

\subsection{Main Result: A Lower Bound on Stochastic First-Order Oracle Complexity}
The following theorem establishes a complexity lower bound for stochastic first-order bilevel algorithms.
\begin{theorem}\label{th:sclower}
    For any $L_f, L_g, \mu, \Delta, \epsilon > 0$ satisfying $\kappa = L_g / \mu \ge 1$ and $\frac{\Delta}{L_f}=\oo(1)$, there exist functions 
$f : \X \times \Y \to \mathbb{R}$ and 
$g : \X \times \Y \to \mathbb{R}$  
such that $\{f, g\} \in \mathcal{F}(L_f, L_g, \mu, \Delta)$ for some $\X\subset \mathbb{R}^m$ and $\Y\subset \mathbb{R}^n$, and stochastic first-order oracles $O$ for both $f$ and $g$ such that for any first-order bilevel algorithm of the form in \Cref{de:alc}, in order to find an $\epsilon$-accurate stationary point $\x$ such that  
$ \mathbb{E} \left [L_h\|\pp_\X[\x-(1/L_h)\nabla H(\x)] - \x\|_2  \right] < \epsilon,$
the algorithm must use at least 
\begin{align}\label{eq:lowerstoc}
   \Omega\left(\frac{ L_f^3\Delta \kappa^{5/2}\sigma^2 }{L^2_g \epsilon^4} \right)
\end{align}
stochastic oracle calls, where $L_h$ is the smoothness parameter of the hyper-objective $H(\x)$.
\end{theorem}
The proof outline is provided in \Cref{sec:proofstoc}. 
In the stochastic setting, \citealt{arjevani2023lower} establish a lower bound of $\Omega(1/\epsilon^{4})$ for smooth nonconvex optimization, and \citealt{li2021complexity} prove a lower bound of $\Omega(\kappa^{1/3}/\epsilon^{4})$ for smooth nonconvex-strongly-concave min-max optimization. 
For smooth nonconvex-strongly-convex bilevel optimization, we obtain in \Cref{th:sclower} a significantly larger lower bound of $\Omega(\kappa^{5/2}/\epsilon^{4})$. 
To the best of our knowledge, this is the first lower-bound result for stochastic bilevel optimization showing that the nonconvex-strongly-convex bilevel optimization is strictly more challenging than both smooth nonconvex optimization and smooth nonconvex-strongly-concave min-max optimization in the stochastic setting. 

In what follows, we discuss connections and comparisons with existing lower bounds.

\vspace{-0.3cm}
\begin{enumerate}[(1)]
    \item  \citet{kwon2024complexity} establish a lower bound of 
$\Omega(\epsilon^{-6})$ for bilevel optimization under a so-called 
$\y^*$-aware stochastic first order oracle with bounded variance. 
Their hard instance is constructed as
\[
f(\x;y) = y, \; 
g(\x;y) = (y - F(\x))^2, \,\x \in \mathbb{R}^{\epsilon^{-2}}, y \in \mathbb{R}
\]
where function 
{\small$F(\x)
= \epsilon^2 \sum_{i=1}^{\epsilon^{-2}}
\big[\, \Psi(-x_{i-1}) \Phi(-x_i) - \Psi(x_{i-1}) \Phi(x_i) \,\big].$}
It can be verified that $|F(\x)| = \oo(1)$, and therefore
\begin{align*}
\|\nabla^2_{\x,y} g(\x;y)\|_2
= \oo\!\left(\epsilon^2 \sqrt{\epsilon^{-2}}\right)
= \oo(\epsilon).
\end{align*}
In addition, their $\y^*$-aware oracle requires $\|y - y^*\| = \oo(\epsilon)$, such that
$|g(\x;y)|$ is of order $\oo(\epsilon)$. 
These conditions can be approximately  satisfied in our construction by choosing 
$L_g = \oo(\epsilon)$, since 
$\|\nabla^2_{\x,\widetilde\y} g(\x;\widetilde\y)\|_2 = L_g$ and both $\x$ and $\widetilde{\y}$ are bounded. 
By this choice, our \Cref{th:sclower} also yields a lower bound of order $\Omega(\epsilon^{-6})$. 

In contrast, a more standard and practically relevant setting assumes 
$L_g, L_f = \Theta(1)$, independent of $\epsilon$ or the condition number $\kappa$. 
Under this commonly studied regime, obtaining an 
$\Omega(\epsilon^{-6})$ lower bound for bilevel optimization remains an open problem.


\end{enumerate}

\vspace{-0.2cm}

\section{Conclusion}
In this work, we developed new hard instances that establish improved lower bounds for smooth nonconvex and strongly convex bilevel optimization under both deterministic and stochastic first order oracle models. Our results demonstrate that bilevel optimization is fundamentally more challenging than classical single-level and min-max formulations, and they reveal significant separations between the best known upper and lower bounds. These findings highlight that the current theoretical understanding of bilevel optimization is still far from complete.

\section*{Impact Statement}
his paper presents work whose goal is to advance the field of bilevel optimization theory. There are many potential societal consequences of our work, none
which we feel must be specifically highlighted here. 

\section*{Acknowledgments}
 K. Ji was partially supported by NSF grants CCF-2311274 and ECCS-2326592.


\bibliography{example_paper}
\bibliographystyle{icml2026}

\newpage
\appendix
\onecolumn
{\bf \Large Appendix}

\section{Some Discussions on Future Works}
There are several promising directions for future research. First, even for the simplified and practically meaningful setting in which the lower level function is quadratic, the optimal complexity remains open. We suggest that closing these gaps may require first studying this simpler yet meaningful quadratic setting. Moreover, 
our constructions suggest that sharper lower and upper bounds may be obtained by designing algorithms that exploit higher-order structure of the  lower-level function. Second, closing the large gaps between the existing upper bounds and our lower bounds, especially the gap of order $\kappa^{2}$ in the deterministic case and the dependence on $\epsilon$ in the stochastic case, represents an important challenge. Third, another compelling direction is to investigate whether an $\Omega(\epsilon^{-6})$ lower bound can be achieved under the standard regime where the smoothness constants $L_f$ and $L_g$ are $\Theta(1)$, a question that remains unresolved. Finally, extending the lower bound framework to broader variants of bilevel optimization, including settings with constraints, approximate inner solvers, or distributed architectures, may deepen the understanding of the fundamental limits of bilevel learning.

Overall, we hope that the insights developed in this paper serve as a starting point for further studies toward a complete theory of the computational complexity of bilevel optimization.

\section{Analysis and Proof Outline for Stochastic Lower Bound}\label{sec:proofstoc}
We use the following construction $\left\{f_{sc}(\x;\widetilde\y),g_{sc}(\x;\widetilde\y)\right\}$ as the hard instance in the stochastic setting. For any $\x\in\C^{T}_{r_x\lambda/n}$ and $\widetilde\y\in\C^{n(T+1)}_{r_y\lambda}$, 
\begin{align}\label{obj:construction_stoc}
f_{sc}(\x;\mathbf{\widetilde y}) &= \sum_{i=1}^T \frac{\lambda^2L_f}{L}\left[
\Psi\Big(-\frac{C_l}{\lambda}y^{(i-1)}_n\Big)\Phi\Big(-\frac{C_r}{\lambda}y^{(i)}_1\Big) - \Psi\Big(\frac{C_l}{\lambda}y^{(i-1)}_n\Big)\Phi\Big(\frac{C_r}{\lambda}y_1^{(i)}\Big) 
\right] 
\nonumber  \\
g_{sc}(\x;\mathbf{\widetilde y}) &= \sum_{i=0}^T \Big[ \frac{L_gn^2}{2(4n^2+1)}
(\y^{(i)})^\top \Big(\frac{1}{n^2}I_n + A\Big)\y^{(i)} - L_g(\bb_x^{(i)})^\top\y^{(i)}
\Big],
\end{align}
where $r_x$ and $r_y$ are positive numerical constants from the hypercube sizes, chosen such that 
$r_y \ge 10 r_x$ and $r_x > \frac{1}{\widetilde C}$, where $\widetilde C > 0$ is the numerical constant defined in \cref{eq:tildeC}. 
The constants $C_l, C_r,$ and $L$ are the same as in the deterministic setting. 
The parameter $\lambda$ is selected to satisfy $\lambda \sqrt{T} = \oo(1)$, and its exact form will be specified later. Recall that $x_0=\frac{\lambda}{\widetilde C n}<\frac{r_x\lambda}{n}\in \C^{1}_{r_x\lambda/n}$. 




The following lemma shows that, with appropriately chosen $r_x$ and $r_y$, the lower-level minimizer $\widetilde{\y}^*$ lies within the selected bounded domain.
\begin{lemma}\label{le:yindomain}
If $r_y\geq 10r_x$, the lower-level minimizer $\widetilde\y^*$ of the instance in \cref{obj:construction_stoc} satisfies $\widetilde\y^*\in \C^{n(T+1)}_{r_y\lambda}$.
\end{lemma}
Building on \Cref{le:yindomain}, we establish the following lemma, which provides several properties of the instance in \cref{obj:construction_stoc} that will be used in the proof of the main theorem.
\begin{lemma}\label{le:propsc}
  Suppose $r_y\geq 10r_x$, $r_x> \frac{1}{\widetilde C}$, and $\lambda\sqrt{T}=\oo(1)$. The functions $f_{sc}$ and $g_{sc}$ satisfy:
    \begin{enumerate}[(a)]
    \item $f_{sc}$ and $g_{sc}$ satisfy all items $1$-$4$ in \Cref{def:fc}.
        \item $H_{sc}(\mathbf{0})-\min_\x H_{sc}(\x)\leq \frac{12\lambda^2L_fT}{L}$.
        \item $H_{sc}(\x)$ is $L_h$-smooth with $L_h=\frac{c_0 n^2L_f}{L}$ for some numerical constant $c_0$.
        \item For any $(\x,\widetilde\y)\in\C^{T}_{r_x\lambda/n},\times\C^{n(T+1)}_{r_y\lambda}$, we have $\|\nabla_{\widetilde\y}f_{sc}(\x;\widetilde\y)\|_\infty\leq  \frac{c_1\lambda L_f}{L}, \|\nabla_{\widetilde\y}g_{sc}(\x;\widetilde\y)\|_\infty\leq 2 L_gr_y\lambda$, and  $\|\nabla_\x g_{sc}(\x;\widetilde\y)\|_\infty\leq L_g r_y\lambda$,  where $c_1$ is a numerical constant.
    \end{enumerate}
\end{lemma}
Similarly to \Cref{le:grad}, we then provide a lower bound of the hyper-gradient norm when the algorithm has not yet reached the end of the chain. 
\begin{lemma}\label{le:scnorm}
   Suppose $r_x> \frac{1}{\widetilde C}$. If $x_i< \frac{\lambda}{\widetilde C n}$ for some $i\leq T$, then, we have 
    \begin{align*}
        L_h\|\pp_\X[\x-(1/L_h)\nabla H_{sc}(\x)] - \x\|_2 \geq  \frac{c_2L_f n \lambda}{L},
    \end{align*}
    where $c_2>0$ is a numerical constant.
\end{lemma}

Based on all the above auxiliary lemmas, we begin to prove our main theorem. 
\begin{proof}[{\bf Proof of \Cref{th:sclower}}]
Based on part~(d) of \Cref{le:propsc}, we now construct a probability-$p$ zero-chain following the approach of \citealt{arjevani2023lower}, with a slight modification. 
In \citealt{arjevani2023lower}, the key idea is to perturb the gradient only at the next coordinate to be activated, so that this coordinate is revealed with probability $p$.  
For our zero-chain given in \cref{subspaces},  let $i^*\in\big\{n+1,...,(T+1)n\big\}$ be the next coordinate to activate. 
Thus, we can define the stochastic gradient as follows.
\begin{itemize}
    \item When $i^* \text{ mod } n \neq 1$, perturb the gradients at the coordinate $i=i^*$ as $\frac{\xi}{p} \frac{\partial g_{sc}(\x;\widetilde\y)}{\partial \widetilde y_{i}}$,
    where $\xi \sim \mathrm{Bernoulli}(p)$. The gradients at all other coordinates remain unchanged and receive no perturbation.

    \item When $i^* \text{ mod } n = 1$, perturb the gradients at the coordinate $i=i^*$ as
    $\frac{\xi}{p} \frac{\partial g_{sc}(\x;\widetilde\y)}{x_j}$, where $j=(i^*-1)/n$ and $\xi \sim \mathrm{Bernoulli}(p)$. The gradients at all other coordinates remain unchanged and receive no perturbation.
\end{itemize}
Note that in the above stochastic oracles, we {\bf do not} perturb the gradients of $f$. It can be verified that 
the stochastic gradients defined above are unbiased. 
Using \Cref{le:pchain} together with part~(d) of \Cref{le:propsc}, 
we conclude that our construction in \cref{obj:construction_stoc}, equipped with these stochastic oracles, forms a probability-$p$ zero-chain, and the variance of the oracles is bounded by
\begin{align*}
   c_3  L^2_g  \lambda^2 \left(\frac{1-p}{p}\right),
\end{align*}
where the bound follows from $(d)$ of \Cref{le:propsc}, and $c_3$ is a positive numerical constant. Thus, to ensure the variance is bounded by $\sigma^2$, it suffices to choose 
\begin{align}\label{eq:valuep}
p=\min \left\{1, c_3 \frac{L^2_g  \lambda^2}{\sigma^2}\right\}.
\end{align}
Then, based on \Cref{le:p_chain} and the stochastic oracles constructed above, we have that with probability $1-\delta$, $x_T=0$ if 
\begin{align}\label{eq:tzero}
    t\leq \frac{(n-1)T-1-\log(\frac{1}{\delta})}{2p}.
\end{align}
Based on the choice of $p$ in \cref{eq:valuep}, we have 
\begin{align*}
    \frac{(n-1)T-1-\log(\frac{1}{\delta})}{2p} \geq \frac{((n-1)T-1-\log(\frac{1}{\delta}))\sigma^2}{2c_3 L^2_g \lambda^2},
\end{align*}
which, together with \cref{eq:tzero}, yields that with probability $1-\delta$, $x_T=0$ for all 
\begin{align*}
    t\leq \frac{((n-1)T-1-\log(\frac{1}{\delta}))\sigma^2}{2c_3L^2_g \lambda^2}.
\end{align*}
This, with \Cref{le:scnorm}, implies that with probability $1-\delta$, $x_T=0$ for all $t\leq \frac{((n-1)T-1-\log(\frac{1}{\delta}))\sigma^2}{2c_3 L^2_g   \lambda^2}$, and hence 
    \begin{align*}
        L_h\|\pp_\X[\x^t-(1/L_h)\nabla H_{sc}(\x^t)] - \x^t\|_2 \geq  \frac{c_2L_f n \lambda}{L},
    \end{align*}
which, by setting $\lambda = \frac{2L\epsilon}{c_2L_fn}$, yields that $ L_h\|\pp_\X[\x^t-(1/L_h)\nabla H_{sc}(\x^t)] - \x^t\|_2\geq 2\epsilon.$ Set $\delta=\frac{1}{2}$. Then,  for all $t\leq \frac{((n-1)T-1-\log(\frac{1}{\delta}))\sigma^2}{2c_3 L^2_g \lambda^2}$, 
\begin{align*}
    \mathbb{E} \left [L_h\|\pp_\X[\x^t-(1/L_h)\nabla H_{sc}(\x^t)] - \x^t\|_2  \right] \geq \frac{1}{2} (2\epsilon) =\epsilon.
\end{align*}
Based on $(b)$ of \Cref{le:propsc}, we have $\frac{12\lambda^2 L_f T}{L}=\Delta$, which implies that $T= \frac{\Delta L}{12\lambda^2 L_f}$. 
Thus, to achieve an $\epsilon$-accurate stationary point, the algorithm must use at least 
\begin{align*}
    \Omega \left(\frac{nT\sigma^2}{L^2_g \lambda^2} \right) = \Omega \left(\frac{n\Delta\sigma^2}{L_fL^2_g \lambda^4} \right) = \Omega\left(\frac{n^5 L_f^3\Delta\sigma^2 }{L^2_g \epsilon^4} \right),
\end{align*}
    which, together with $n=\sqrt{\kappa}$, finishes the proof. 
\end{proof}


\allowdisplaybreaks
\appendix
\vspace{0.2cm}


\section{Proofs for Deterministic Lower Bound}
\subsection{Proof of \Cref{le:M}}
It is straightforward to verify that $c \le S_{1,1} = 1 \le C$, so the claim holds for $n=1$. 
For the remainder of the proof, we assume $n \ge 2$. 
Set $s:=1/n^{2}$.  
The eigenpairs of $A$ are
\[
\mu_k = 2\Bigl(1-\cos\Bigl(\frac{(k-1)\pi}{n}\Bigr)\Bigr),\qquad
k=1,\dots,n,
\]
with orthonormal eigenvectors
\[
q_1(j)=\frac{1}{\sqrt{n}},\qquad
q_k(j)=\sqrt{\tfrac{2}{n}}
\cos\Bigl(\frac{(k-1)(j-\tfrac12)\pi}{n}\Bigr),\ k\ge2.
\]
Thus
\[
A = Q \Lambda Q^\top,\qquad
\Lambda=\operatorname{diag}(\mu_1,\dots,\mu_n),\qquad
Q=[q_1\,\dots\,q_n].
\]
Hence
\[
S = (A+sI_n)^{-1}
= Q(\Lambda+sI_n)^{-1}Q^\top
= \sum_{k=1}^n \frac{1}{\mu_k+s}\, q_k q_k^\top,
\]
so we can express $S_{i,j}$ as 
\[
S_{i,j} = \sum_{k=1}^n \frac{q_k(i)q_k(j)}{\mu_k+s}.
\]
Note that  
$\frac{q_1(i)q_1(j)}{s}
= \frac{1/n}{1/n^2}
= n.$ Thus, we have  
\[
S_{i,j} = n + R_{i,j},\qquad
R_{i,j} := \sum_{k=2}^n \frac{q_k(i)q_k(j)}{\mu_k+s}.
\]

\paragraph{Higher eigenmodes.}
Because $|q_k(\cdot)|\le\sqrt{2/n}$,
\[
|q_k(i)q_k(j)|\le \frac{2}{n}.
\]
Also for $k\ge2$,
\[
\mu_k = 2(1-\cos(\tfrac{(k-1)\pi}{n}))
\ge \frac{4(k-1)^2}{n^2},
\]
so
\[
\frac{1}{\mu_k+s}
\le \frac{n^2}{4(k-1)^2}.
\]
Thus
\[
|R_{i,j}|
\le \sum_{k=2}^n \frac{2}{n}\cdot\frac{n^2}{4(k-1)^2}
= \frac{n}{2}\sum_{m=1}^{n-1}\frac{1}{m^2}
\le \frac{\pi^2}{12}\,n.
\]

\paragraph{Final bounds.}
\[
S_{n,n} = n+R_{n,n},\ R_{n,n}\ge 0,
\qquad
S_{1,n} = n+R_{1,n},\ |R_{1,n}|\le\frac{\pi^2}{12}n.
\]
Hence for all $n\ge2$,
\[
\Bigl(1-\frac{\pi^2}{12}\Bigr)n
\le S_{1,n}, S_{n,n}
\le \Bigl(1+\frac{\pi^2}{12}\Bigr)n.
\]
Then, the proof is complete. 

\subsection{Proof of \Cref{le:grad}}
    Note that $x_0=\frac{\lambda}{C_lM_{n,n}} = \frac{\lambda}{\widetilde C n}$. Since $|x_0|\geq \frac{\lambda}{\widetilde C n}$ and $|x_i|<\frac{\lambda}{\widetilde C n}$, we can find some $0<j\leq i$ such that $|x_{j-1}|\geq \frac{\lambda}{\widetilde C n}$ and $|x_{j}|< \frac{\lambda}{\widetilde C n}$.  Thus, look at 
    \begin{align*}
        \frac{\partial H(\x)}{\partial x_j} =& -\frac{\lambda L_f\widetilde C n }{L}\left[
\Psi\Big(-\frac{\widetilde C n}{\lambda}x_{j-1}\Big)\Phi^\prime\Big(-\frac{\widetilde C n}{\lambda}x_j\Big) + \Psi\Big(\frac{\widetilde C n}{\lambda}x_{j-1}\Big)\Phi^\prime\Big(\frac{\widetilde C n}{\lambda}x_j\Big)
\right] \nonumber
\\ &- \frac{\lambda L_f \widetilde C n }{L}\left[
\Psi^\prime\Big(-\frac{\widetilde C n}{\lambda}x_{j}\Big)\Phi\Big(-\frac{\widetilde C n}{\lambda}x_{j+1}\Big) + \Psi^\prime\Big(\frac{\widetilde C n}{\lambda}x_{j}\Big)\Phi\Big(\frac{\widetilde C n}{\lambda}x_{j+1}\Big)
\right],
    \end{align*}
which, in conjunction with \Cref{le:phipsi} (items 2 and 4), implies that 
\begin{align*}
   \|\nabla H(\x)\|_2 \geq  \Big | \frac{\partial H(\x)}{\partial x_j}\Big | \geq \frac{\lambda L_f \widetilde C n}{L}.
\end{align*}
Then, the proof is complete. 

\subsection{Proof of \Cref{le:function}}
    First note that 
    \begin{align}
       H(\mathbf{0})= \frac{\lambda^2 L_f}{L}   \left [ \left(\Psi\Big(-\frac{C_l}{\lambda}M_{n,n}x_0\Big) - \Psi\Big(\frac{C_l}{\lambda}M_{n,n}x_0\Big)\right)\Phi(0)
       \right ]  \leq  0, 
    \end{align}
where the inequality follows because $\frac{C_l}{\lambda}M_{n,n}x_0\geq 0$ and from the definitions of $\Psi$ and $\Psi$ functions in \cref{eq:phipsi}. Furthermore, based on \Cref{le:phipsi} (item 4), we have that 
    \begin{align}
        H(\x) \geq -\frac{\lambda^2L_f}{L}\sum_{i=1}^T \Psi\Big(\frac{C_l}{\lambda}M_{n,n}x_{i-1}\Big)\Phi\Big(\frac{C_r}{\lambda}M_{1,n}x_i\Big) \geq   -\frac{12\lambda^2L_fT}{L},
    \end{align}
    which, combined with $H(\mathbf{0})\leq 0$, implies that 
    \begin{align*}
        H(\mathbf{0}) - \inf_{\x}H(\x) \leq \frac{12\lambda^2L_fT}{L},
    \end{align*}
    which finishes the proof. 

\section{Proofs for Stochastic Lower Bound}
\subsection{Auxiliary Lemmas}

For a probability-$p$ zero-chain, at each iteration, a new coordinate is discovered with probability at most $p$. Therefore, it takes at least $1/p$ steps in expectation to activate a new coordinate.  
The following lemma, adapted from \citealt{arjevani2023lower, li2021complexity}, shows that 
at least $\Omega(T/p)$ iterations are required to reach the end of a probability-$p$ zero-chain.
\begin{lemma}[Lemma~1 in {\citealt{arjevani2023lower}}]\label{le:p_chain}
Let $f : \mathcal{X} \to \mathbb{R}$, where $\mathcal{X} \subset \mathbb{R}^T$ satisfies  
$\operatorname{supp}\big(P_{\mathcal{X}}(\x)\big) = \operatorname{supp}(\x),\;\forall\, \x \in \mathbb{R}^T,$
and suppose $f$ is a probability-$p$ zero-chain with a stochastic first-order oracle.
Then, for any first-order algorithm, with probability at least $1 - \delta$, the $T$-th coordinate of $\x$ at the $t^{th}$ iteration,  satisfies
\[
x_T^{\,t} = 0, \qquad 
\forall\, t \le \frac{T - \log(1/\delta)}{2p}.
\]
\end{lemma}

\begin{lemma}\label{lem:monotone-column}
Recall $ S := \bigl(A + \tfrac{1}{n^{2}} I_n\bigr)^{-1}
$. For every $i=1,\dots,n$,
\[
S_{1,n} \le S_{i,n} \le S_{n,n}.
\]

\end{lemma}

\begin{proof}

Define $B := A + \frac{1}{n^{2}} I_n,
$ and $\vv\in\mathbb{R}^n$ the last column of $S$, i.e.,
$
\vv:=S_{\cdot,n}, v_i := S_{i,n},\ i=1,\dots,n.$ 
Since $\vv$ is the last column of $S=B^{-1}$, it solves the linear system
\[
B\vv = \e_n.
\]
Writing this componentwise, we obtain
\begin{align*}
(1+n^{-2})v_1 - v_2 &= 0,\\
-v_{i-1} + (2+n^{-2})v_i - v_{i+1} &= 0,\qquad i=2,\dots,n-1,\\
-v_{n-1} + (1+n^{-2})v_n &= 1.
\end{align*}
Define the forward differences
\[
d_i := v_{i+1} - v_i,\qquad i=1,\dots,n-1.
\]
We next derive a system of equations for $\dd=(d_1,\dots,d_{n-1})^\top. $ 
For $i=2,\dots,n-2$, subtracting the equation at index $i$ from that at index $i+1$ gives
\[
\bigl(-v_{i} + (2+n^{-2})v_{i+1} - v_{i+2}\bigr)
-
\bigl(-v_{i-1} + (2+n^{-2})v_i - v_{i+1}\bigr) = 0,
\]
which can be rewritten as
\[
-d_{i-1} + (2+n^{-2}) d_i - d_{i+1} = 0,
\qquad i=2,\dots,n-2.
\]
From the first equation, we obtain
\[
(1+n^{-2})v_1 - v_2 = 0
\quad\Longrightarrow\quad
(2+n^{-2})d_1 - d_2 = 0.
\]
From the last equation, we obtain
\[
-v_{n-1} + (1+n^{-2})v_n = 1
\quad\Longrightarrow\quad
-d_{n-2} + (2+n^{-2}) d_{n-1} = 1.
\]
Thus, the vector $\dd=(d_1,\dots,d_{n-1})^\top$ satisfies a tri-diagonal linear system
\[
B' \dd = \e_{n-1},
\]
where $B'\in\mathbb{R}^{(n-1)\times(n-1)}$ is the symmetric tri-diagonal matrix
\[
B' =
\begin{bmatrix}
2+n^{-2} & -1        \\
-1        & 2+n^{-2} & -1 \\
          & \ddots   & \ddots & \ddots \\
          &          & -1 & 2+n^{-2}
\end{bmatrix}.
\]
The matrix $B'$ is strictly diagonally dominant with positive diagonal entries
and nonpositive off-diagonal entries, hence an irreducible $M$-matrix \citep{berman1994nonnegative}.
It is therefore positive definite and its inverse is entrywise nonnegative:
\[
(B')^{-1} \ge 0 \quad \text{(entrywise)}.
\]
Since $\dd = (B')^{-1} \e_{n-1}$, we obtain
\[
d_i \ge 0,\qquad i=1,\dots,n-1.
\]
Equivalently,
\[
v_{i+1}-v_i = d_i \ge 0
\quad\Longrightarrow\quad
v_1 \le v_2 \le \dots \le v_n.
\]
The inequalities $S_{1,n} \le S_{i,n} \le S_{n,n}$ follow immediately from
this monotonicity.
\end{proof}

\subsection{Proof of \Cref{le:yindomain}}
Note that the minimizers $(\y^{(i)})^*,i=0,...,T$ take the forms of 
\begin{align*}
    (\y^{(i)})^* = \underbrace{\frac{4n^2+1}{n^2}\left(\frac{1}{n^2}I_n + A\right)^{-1}}_{M}\bb_{x}^{(i)}.
\end{align*}
Combining \Cref{le:M} and \Cref{lem:monotone-column}, we have for all $i=1,...,n$,
\begin{align*}
  4cn\leq  M_{i,n}\leq 5Cn,
\end{align*}
where $c=1-\frac{\pi^2}{12}$ and $C=1+\frac{\pi^2}{12}$. Thus, we have 
\begin{align*}
\|(\y^{(i)})^*\|_{\infty}\leq 5Cn|x_i|\leq 5C r_x\lambda< 10r_x\lambda < r_y\lambda,
\end{align*}
which finishes the proof. 

\subsection{Proof of \Cref{le:propsc}}
   The proof of (a) is identical to the deterministic case. 
The proof of (b) follows the same reasoning as in \Cref{le:function}. 
To establish (c), recall that 
    \begin{align}
        H_{sc}(\x)=& \sum_{i=1}^T \frac{\lambda^2L_f}{L}\left[
\Psi\Big(-\frac{C_l}{\lambda}M_{n,n}x_{i-1}\Big)\Phi\Big(-\frac{C_r}{\lambda}M_{1,n}x_i\Big) - \Psi\Big(\frac{C_l}{\lambda}M_{n,n}x_{i-1}\Big)\Phi\Big(\frac{C_r}{\lambda}M_{1,n}x_i\Big)
\right]. \nonumber 
    \end{align}
Then one can verify that $\nabla^{2} H_{sc}(\x)$ is a tri-diagonal matrix whose entries are all of order $\oo\!\left(\tfrac{L_f n^{2}}{L}\right)$. Consequently, 
$\big\|\nabla^{2} H_{sc}(\x)\big\|_{2}
= \oo\!\left(\tfrac{L_f n^{2}}{L}\right).$ To prove $(d)$, note that each coordinate of $\nabla_{\widetilde\y} f_{sc}(\x;\widetilde\y)$ takes an order of $\oo(\frac{\lambda L_f}{L})$, and hence $\|\nabla_{\widetilde\y} f_{sc}(\x;\widetilde\y)\|_\infty=\oo(\frac{\lambda L_f}{L})$.

For $\nabla_{\widetilde\y}g_{sc}(\x;\widetilde\y)$, note that 
\begin{align*}
   \left\| \frac{\partial g_{sc}(\x;\widetilde\y)}{\partial \y^{(i)}} \right\|_\infty = & \left\|\frac{L_gn^2}{4n^2+1}
 \Big(\frac{1}{n^2}I_n + A\Big)\y^{(i)} - L_g \bb_x^{(i)} \right\|_\infty \nonumber
 \\ \leq & \frac{L_g}{4}\left\| \frac{1}{n^2}I_n + A \right\|_\infty \left\|\y^{(i)}\right\|_\infty + L_g |x_i| \nonumber
 \\\leq & \frac{L_g}{4}\left(\frac{1}{n^2}+4\right)r_y\lambda + \frac{L_g r_x\lambda}{n} \nonumber
 \\\leq & \frac{5}{4} L_g r_y \lambda + \frac{1}{10}L_gr_y\lambda \leq 2 L_gr_y\lambda,
\end{align*}
which holds for all $i=0,...,T$. This implies that $\|\nabla_{\widetilde\y}g_{sc}(\x;\widetilde\y)\|_\infty\leq 2L_g r_y\lambda$.

For $\|\nabla_\x g_{sc}(\x;\widetilde\y)\|_\infty$, note that
\begin{align*}
    \left | \frac{\partial g_{sc}(\x;\widetilde\y) }{\partial x_i} \right| = L_g |y_n^{(i)}|\leq L_g r_y\lambda,
\end{align*}
which yields that $\|\nabla_\x g_{sc}(\x;\widetilde\y)\|_\infty\leq L_g r_y\lambda$. 
Then, the proof is complete. 

\subsection{Proof of \Cref{le:scnorm}}

Note that $x_0=\frac{\lambda}{C_lM_{n,n}} = \frac{\lambda}{\widetilde C n}$. Since $|x_0|\geq \frac{\lambda}{\widetilde C n}$ and $|x_i|<\frac{\lambda}{\widetilde C n}$, we can find some $0<j\leq i$ such that $|x_{j-1}|\geq \frac{\lambda}{\widetilde C n}$ and $|x_{j}|< \frac{\lambda}{\widetilde C n}$.  Thus, look at 
    \begin{align}
         \frac{\partial H_{sc}(\x)}{\partial x_j} =& -\frac{\lambda L_f\widetilde C n }{L}\left[
\Psi\Big(-\frac{\widetilde C n}{\lambda}x_{j-1}\Big)\Phi^\prime\Big(-\frac{\widetilde C n}{\lambda}x_j\Big) + \Psi\Big(\frac{\widetilde C n}{\lambda}x_{j-1}\Big)\Phi^\prime\Big(\frac{\widetilde C n}{\lambda}x_j\Big)
\right] \nonumber
\\ &- \frac{\lambda L_f \widetilde C n }{L}\left[
\Psi^\prime\Big(-\frac{\widetilde C n}{\lambda}x_{j}\Big)\Phi\Big(-\frac{\widetilde C n}{\lambda}x_{j+1}\Big) + \Psi^\prime\Big(\frac{\widetilde C n}{\lambda}x_{j}\Big)\Phi\Big(\frac{\widetilde C n}{\lambda}x_{j+1}\Big)
\right].
    \end{align}
Then, if $\big|x_j - (1/L_h)\frac{\partial H_{sc}(\x)}{\partial x_j}\big|\leq r_x\lambda/n$, then we have 
\begin{align*}
            L_h\|\pp_\X[\x-(1/L_h)\nabla H_{sc}(\x)] - \x\|_2 \geq \left | \frac{\partial H_{sc}(\x)}{\partial x_j}\right |\geq \frac{\lambda L_f \widetilde C n}{L}. 
\end{align*}
Otherwise, i.e., $\big|x_j - (1/L_h)\frac{\partial H_{sc}(\x)}{\partial x_j}\big|> r_x\lambda/n$, we have 
\begin{align*}
    L_h\|\pp_\X[\x-(1/L_h)\nabla H_{sc}(\x)] - \x\|_2  |\geq &
    L_h \left | \pp_{\C_{r_x\lambda/n}^1}\left[x_j-(1/L_h)\frac{\partial H_{sc}(\x)}{\partial x_j}\right] - x_j \right |
    \\ \geq & L_h \left(\frac{r_x\lambda}{n} -|x_j|\right) \geq L_h\left(\frac{r_x\lambda}{n} - \frac{\lambda}{\widetilde C n}\right) 
    \\\overset{(i)}\geq & \frac{c_0 n^2L_f}{L} \left(r_x - \frac{1}{\widetilde C}\right)  \frac{\lambda}{n} = \frac{c_0 L_f }{L} \left(r_x - \frac{1}{\widetilde C}\right) n\lambda,
\end{align*}
where $(i)$ follows from $(c)$ of \Cref{le:propsc}. Combining the above two cases completes the proof.

\end{document}